\documentclass{article}
\usepackage{others/iclr2027_conference,times}
\usepackage{graphicx}
\usepackage{float}
\usepackage{wrapfig}

\usepackage{amsmath,amsfonts,bm}

\def\eqref#1{equation~\ref{#1}}

\def\1{\bm{1}}

\DeclareMathAlphabet{\mathsfit}{\encodingdefault}{\sfdefault}{m}{sl}
\SetMathAlphabet{\mathsfit}{bold}{\encodingdefault}{\sfdefault}{bx}{n}

\usepackage{amsthm}
\usepackage{hyperref}
\usepackage{url}
\usepackage[percent]{overpic}
\usepackage{caption}
\usepackage{setspace}
\usepackage[most]{tcolorbox}

\newtcolorbox{researchquestion}{
    colback=blue!4,
    colframe=blue!35,
    boxrule=0.5pt,
    arc=3pt,
    left=7pt,
    right=7pt,
    top=4pt,
    bottom=4pt
}

\newtheorem{theorem}{Theorem}[section]
\newtheorem{lemma}[theorem]{Lemma}
\newcommand{\spar}[1]{\textbf{#1}.}

\title{Neural Dynamics as the Composition of Quantized Units}

\iclrfinalcopy
\renewcommand{\headrulewidth}{0.4pt}
\author{Jacopo Minniti\textsuperscript{1}\thanks{Work done during an internship at Sakana AI.}\thanks{Correspondence: \texttt{jacopo@uni.minerva.edu}.}
\quad Aravinth Kulanthaivelu\textsuperscript{2}
\quad Richard Sproat\textsuperscript{2}\\[0.7em]
\normalsize\textsuperscript{1}Minerva University
\qquad\textsuperscript{2}Sakana AI}

\begin{document}
\maketitle

\begin{abstract}
Deep learning is commonly interpreted at two levels: the macroscopic, through aggregate trends in loss summarized by scaling laws, and the microscopic, through neurons, features, and circuits. A central challenge is understanding how these levels connect, so that we can explain how elementary computations compose and collectively shape macroscopic behavior. To this end, we study an intermediate abstraction in which training is described as the ordered acquisition of \emph{quanta}: reusable computations acquired suddenly and binary-activated across examples to reduce loss. By approximating population-gradient updates, we derive quanta's acquisition dynamics. This yields an acquisition priority governed by \emph{demand}, how frequently a computation is required across examples, and \emph{conditional complexity}, how difficult that computation is to acquire given those already available. In a Boolean compositional task, we derive predictions for acquisition order and show how staggered discrete acquisitions can produce smooth aggregate loss and, under certain geometries of quanta composition, give rise to scaling laws. We then train a Transformer to map numerals to English number names and recover candidate quanta from its checkpoint trajectory. From these units, we construct a model that preserves much of the Transformer's behavior while exposing interpretable latent computations and acquisition dynamics consistent with the theory. Separately, the quanta structure can serve as training targets to improve transformer generalization. Together, these results suggest the quanta abstraction can provide useful computational atoms for studying a variety of macroscopic phenomena.
\end{abstract}

\section{Introduction}
\label{sec:introduction}

Recent years have brought substantial progress in deep learning theory. Exact solutions for deep linear networks explain learning plateaus and rapid transitions~\citep{saxe2014deep}, while work on multilayer networks shows how parameterization controls the evolution of learned features during training~\citep{yang2021tensor}. Results like these illuminate specific mechanisms, but a theory that connects neural dynamics, representations, and performance across modern architectures has yet to emerge, as analytical tractability continues to depend largely on strong assumptions about linearity, tractable limits, or the data distribution~\citep{simon2026scientific}. Thus, for more complex settings such as large language models (LLMs), a growing body of literature trades analytical precision for a more empirical approach to study learned computations. Mechanistic interpretability has uncovered key insights on the inner workings of LLMs, such as how induction heads are involved in in-context learning~\citep{olsson2022induction} or how sparse features can be used to interpret complex representations~\citep{cunningham2023sae}. Yet, a general theoretical framework connecting these findings is still lacking, as relations among objects such as \textit{features} and \textit{circuits} remain only partially understood, as does how mechanisms established in individual case studies combine into a unified account of model behavior~\citep{gauderis2026compositional}.
\begin{figure*}[t]
    \centering
    \includegraphics[width=0.95\textwidth]{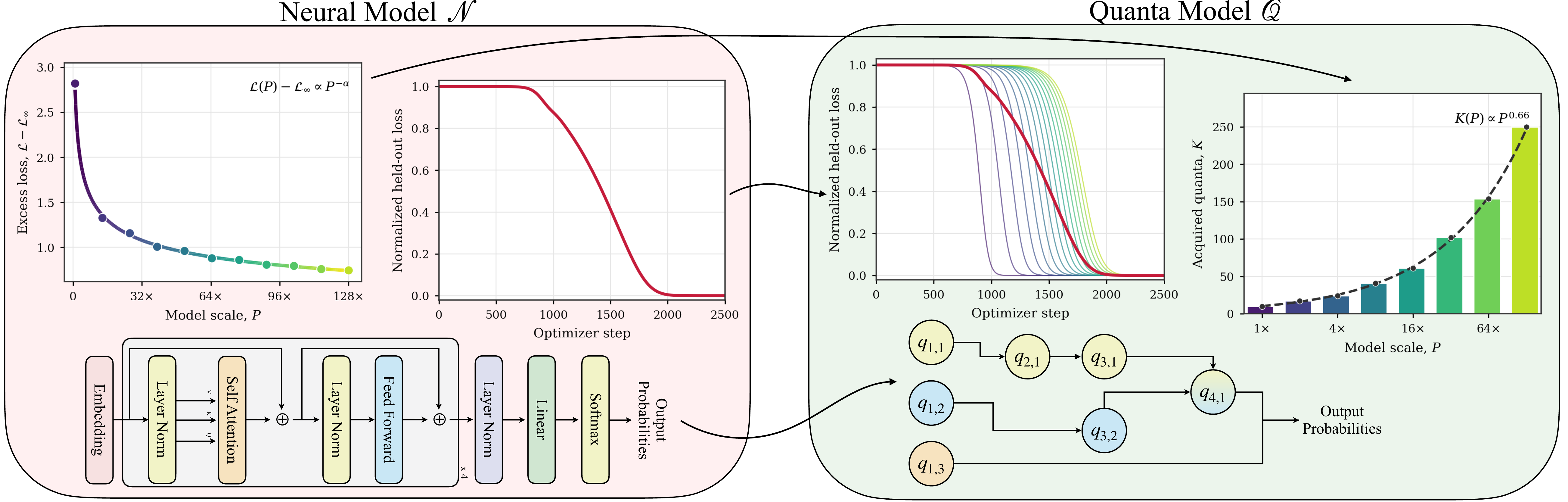}
    \caption{\textbf{Examples of Quanta Descriptions of a Neural Model.} A smooth loss curve can be decomposed into the staggered acquisition of discrete computations. Scaling with model size becomes growth of the acquired frontier, whose unmet demand yields a power law. An opaque neural network can be rewritten as a more interpretable \textit{Q-model} which follows the quanta dynamics.}
    \label{fig:quanta_overview}
\end{figure*}

In parallel, work has begun to synthesize these lines of inquiry into accounts of learning that remain analytically tractable at scale without requiring parameter-specific network descriptions. Doing so requires admittedly strong but productive assumptions about the learning problem. For example, \citet{vanrossem2024universality} reduces sufficiently flexible architectures to idealized encoder--decoder systems, allowing simplified learning dynamics to be derived, while \citet{cagnetta2024hierarchy} assumes a specific hierarchical structure in the data distribution and derives power laws similar to those observed in LLM training. We adopt a similar strategy to ask:

\begin{researchquestion}
\textbf{Research Question.}
How can a model's behavior and learning dynamics be described in terms of reusable computations that are acquired and activated discretely?
\end{researchquestion}

Our choice of discrete units is motivated by a growing body of work in which representations emerge through abrupt transitions even when aggregate loss is smooth~\citep{barak2022hidden,simon2023stepwise} and by the observation that finite-capacity models trained on large datasets benefit from reusing the same learned computation across samples that share a common trait~\citep{olah2020zoom,saxe2022race}. Building directly on the \textit{Quantization Model} proposed by \citet{quanta_model}, we conjecture that much of a neural network's behavior can be approximated by composing reusable computations acquired through sudden events, termed \textit{quanta}. We then approximate gradient descent by dynamics over this space: a quantum tends to be acquired earlier when it is useful to more examples---a property we call its \emph{demand}---and when it is simpler given the currently available computations---its \emph{conditional complexity}. Quanta, their input-dependent execution, and their composition form a \emph{Q-model} that attempts to approximate the source neural model.

\spar{Contributions}
We make three contributions. \textbf{(i)} We formalize functional quantization and derive how demand and conditional complexity order the acquisition of reusable computations. \textbf{(ii)} We connect discrete acquisitions to smooth aggregate loss in a toy Boolean task, analyze how data choices select different posets, and derive scaling laws from their geometry. \textbf{(iii)} Using the \textit{Number Naming} dataset, we introduce a trajectory-based method for recovering Q-models from ordinary neural networks and explore program-derived Q-models as structural priors and removable training targets. Code is available at \href{https://github.com/SakanaAI/quanta}{\texttt{github.com/SakanaAI/quanta}}.

\section{Related Work}
\label{sec:related_work}

\spar{Skills and ordered acquisition}
A recently growing body of work has been concerned with how \emph{skills}, loosely defined as reusable abilities that contribute to predictions across examples, are acquired during training. The Quantization Model links frequency-ranked skill acquisition to emergence and scaling~\citep{quanta_model}. Other accounts explain the order differently: task imbalance or composition produces ``domino'' effects~\citep{liu2025physics}, while broad training can provide prerequisites for a narrow hierarchical skill~\citep{michaud2025narrow}. Skill-It uses dependencies between data-defined skills to choose training mixtures~\citep{chen2023skillit}, and language-model checkpoints show component tasks preceding composite ones~\citep{liu2026implicit}.

\spar{Discrete transitions beneath smooth loss}
Training trajectories can conceal progress or show sharp changes at different resolutions. Self-supervised models acquire representation modes stepwise~\citep{simon2023stepwise}. Induction heads emerge with a jump in in-context learning~\citep{olsson2022induction}, and syntactic attention structure appears during a brief loss drop~\citep{chen2023sudden}. Decomposing language-model loss by example groups reveals breakthroughs obscured by its mean~\citep{kangaslahti2026hidden}. These observations motivate looking below aggregate loss, yet a task-level transition cannot distinguish one reusable computation from several partial or joint acquisitions.

\spar{Features and circuits}
A feature is a learned component in a model's internal representation whose activation varies systematically with an input property and can influence downstream computation. Sparse autoencoders seek such components in activations~\citep{cunningham2023sae}, while attribution-based parameter decomposition seeks simple, sparsely used components in the original weights~\citep{braun2025apd}. Circuits ask how components interact to produce behavior~\citep{olah2020zoom}; circuit tracing studies these interactions on individual prompts using sparse MLP replacements~\citep{ameisen2025circuit}. Causal abstraction uses interventions to test whether a proposed high-level account faithfully captures the original network~\citep{geiger2021causal}.

\spar{Explanations of neural scaling}
Empirical laws relate language-model loss to parameters, data, and compute~\citep{kaplan2020scaling,scaling_laws}. This prompted a search for microscopic explanations of how such regularities emerge from the structure of data and the computations acquired during learning. Explanations based on data statistics invoke manifold resolution and kernel spectra~\citep{bahri2021explaining}, while hierarchical generators show how deep models acquire invariant representations~\citep{cagnetta2024hierarchy} and how rule frequencies can shape classification and next-token learning curves differently~\citep{cagnetta2025learningcurves}. Other work derives scaling laws from the acquisition of discrete skills, showing how frequency-ranked or power-law-distributed skills produce an unmet-loss tail~\citep{quanta_model,nam2024solvable}.
\section{Neural Dynamics in Q-Coordinates}
\label{sec:quantized_model}

Let $\mathcal N_t=(\theta_t,\Phi_N)$ denote a neural model at training time $t$ through its weights and their update rule, $\theta_{t+1}=\Phi_N(\theta_t)$. The model is trained on a distribution $\mathcal D$, with population loss $\mathcal L(\theta)=\mathbb E_{x\sim\mathcal D}[\mathcal L_x(\theta)]$. Its parameter trajectory describes learning exactly, but is too detailed to explain which computations emerge, while its scalar loss is concise, but discards those computations altogether. We seek a description between these extremes capable of following the model in function space while retaining only the changes that matter for its behavior.

\subsection{Functional quantization}
Different examples in a complex dataset often require duplicate work. In hierarchical tasks, the advantage is especially clear, as a reusable intermediate can replace many lower-level patterns and simplify each later computation that consumes it~\citep{cagnetta2024hierarchy}. More generally, a finite-capacity model is often incentivized to learn a computation once and apply it wherever it is useful~\citep{olah2020zoom}.
Let $\mathcal N_t(x)=\sum_{i=1}^M q_i(x,t)$ for every $x\in\mathcal D$ be a decomposition of the model into reusable, composable routines whose precise form we want to characterize. A function is primitive only if this model acquires it as one coordinated functional change. We assume that, at time $t$, the model has either acquired $q_i$ reliably or it has not (\emph{discreteness of acquisition}), and denote the corresponding change point by $\tau_i$. We further assume that $q_i$ is either required or not required to solve an input $x$ (\emph{discreteness of consumption}), and denote the inputs that require it by $A_i$. If $\operatorname{Pa}(i)$ are the predecessors of $i$ in an acyclic execution graph, its contribution is $q_i(x,t)=\mathbf 1\{t\geq\tau_i\wedge x\in A_i\}\,f_i(x,\{q_r(x,t):r\in\operatorname{Pa}(i)\})$.
The quanta abstract the circuits implemented by $\mathcal N_t$, which may themselves be highly nonlinear. We therefore impose no particular functional form on $f_i$. Each $q_i$ quantizes the knowledge available to the model, thus motivating the term \emph{quantum}~\citep{quanta_model}. The functional state of the model at time $t$ is reduced from all of $\theta_t$ to the acquired set $I_t=\{ q_i: \tau_i \leq t \}$ and the quanta required by each input.

\subsection{Learning in Q-coordinates}
\spar{Population-gradient priority}
A faithful Q-description must capture not only a checkpoint, but when and in what order its functions are acquired. For $\mathcal N_t$, we assume the classical gradient-descent update rule $\Phi_N(\theta_t)=\theta_t-\eta_t g_t$, where $g_{x,t}=\nabla_\theta\mathcal L_x(\theta_t)$ and $g_t=\mathbb E_x[g_{x,t}]$. For a sufficiently small step, $\mathcal L_x(\theta_{t+1})-\mathcal L_x(\theta_t)=-\eta_t\langle g_{x,t},g_t\rangle + O(\eta_t^2)$. We call $P(x,t)=\eta_t\langle g_{x,t},g_t\rangle$ the \emph{priority} of example $x$. To first order, this is the example's loss reduction under the population update; we ask how it is shared when many examples require the same missing quantum.

\spar{Demand and conditional complexity}
Consider a missing quantum $q$ and let $A_q$ contain the examples that require it. Its demand is the probability $p_q=\Pr_{x\sim\mathcal D}(x\in A_q)$. From the current acquired set $I_t$, let $v_q(I_t)$ be a unit direction in parameter space that builds $q$. Among the examples in $A_q$, the mean gradient progress along this direction is $s_q(I_t)=\mathbb E[-\langle v_q(I_t),g_{x,t}\rangle\mid x\in A_q]$. The projection $\langle v_q(I_t),\theta_{t+1}-\theta_t\rangle$ equals $-\eta_t\mathbb E_x[\langle v_q(I_t),g_{x,t}\rangle]$ and, assuming negligible projection outside $A_q$, it is approximately $\eta_t p_qs_q(I_t)$. When $q$ captures shared work, their updates reinforce one another in proportion to $p_q$, while unrelated examples contribute no systematic progress. A quantum with greater demand therefore receives more population-gradient signal.

Demand alone would imply that the most demanded missing quantum comes next, treating every function as equally accessible. This is unrealistic: consider a network learning $f(x)=x^8$. Even if it is required on every example, implementing it directly may be much harder than first learning ``squaring" and composing it to obtain $x^8$. Learning one quantum changes the available representation, and therefore the difficulty of learning the next. Let $T_q(I)$ be the time required to acquire quantum $q$ from the current acquired set $I$, and suppose that $d_q(I)$ is the progress needed to make $q$ reliable. While $I$ and the learning rate are approximately constant, the preceding progress relation gives $T_q(I)\simeq d_q(I)/(\eta p_qs_q(I))$. We collect the two sources of difficulty---how far the solution is and how effectively the gradient approaches it---into the \emph{conditional complexity} $\kappa_q(I)=d_q(I)/s_q(I)$. After absorbing the learning rate and a common proportionality constant into the unit of time, $T_q(I)\propto\kappa_q(I)/p_q$. Equivalently, we denote this acquisition priority by $\omega_q(I):=p_q/\kappa_q(I)$.
This yields an approximation of dynamics over weights by dynamics over quanta. Let $A(x)$ represent the quanta required by $x$, and let $\mathcal F(I_t)$ contain those currently accessible from $I_t$. After choosing units so that acquiring one quantum corresponds to one unit of functional progress, the event priorities induced by gradient descent become $P(x,t)\simeq\sum_{q\in A(x)\cap\mathcal F(I_t)}\omega_q(I_t)$. Demand determines how broadly gradient signal is shared, conditional complexity determines how effectively that signal can be converted into a new computation, and their ratio indicates which change occurs first.

\spar{Quanta structure}
Acquiring one computation can change how difficult it is to learn another. Let $r$ and $q$ be two quanta. Along a realized learning route, we write $r\prec_L q$ when $q$, given the current state of the model and data distribution, cannot be acquired before $r$, so $r$ is a necessary prerequisite of $q$. Thus $q$ is accessible only when all of its predecessors have been acquired. These relations are acyclic along a realizable trajectory, and their reflexive and transitive closure $\preceq_L$ therefore makes $(\mathcal Q,\preceq_L)$ the \emph{learning poset}.
In this way, we can approximate training in a high-dimensional parameter space by a trajectory through the feasible states of a poset. The current state $I$ exposes a precise frontier $\mathcal F(I)=\{q\notin I:\operatorname{Pred}(q)\subseteq I\}$ of possible next acquisitions. The poset separates the structural question of what can be learned next, $\Phi_Q(I)=I\cup\{q^\star(I)\}$, from the question of which frontier quantum receives the greatest priority, $q^\star(I)=\arg\max_{q\in\mathcal F(I)}\omega_q(I)$.
Mirroring the neural description $\mathcal N_t=(\theta_t,\Phi_N)$, we define the Q-model as $\mathcal Q_t=(I_t,\Phi_Q)$: its state is the acquired quanta and its update rule specifies the next acquisition, $I_{n+1}=\Phi_Q(I_n)$, where $n$ indexes acquisition events. Ideally, $\mathcal Q$ serves as a change of coordinates for every $x$ and $t$, representing the same functional changes as $\mathcal N$ through training.

\section{Quantized Dynamics and Scaling}
\label{sec:loss_dynamics}

Section~\ref{sec:quantized_model} describes learning through changes in the available computations. We now study how acquisitions compose, why their aggregate looks smooth, and how macroscopic regularities arise.

\subsection{Hierarchical Sparse Parity}
\begin{wrapfigure}[17]{R}{0.50\textwidth}
    \centering
    \includegraphics[width=0.85\linewidth,trim=0 25 0 60,clip]{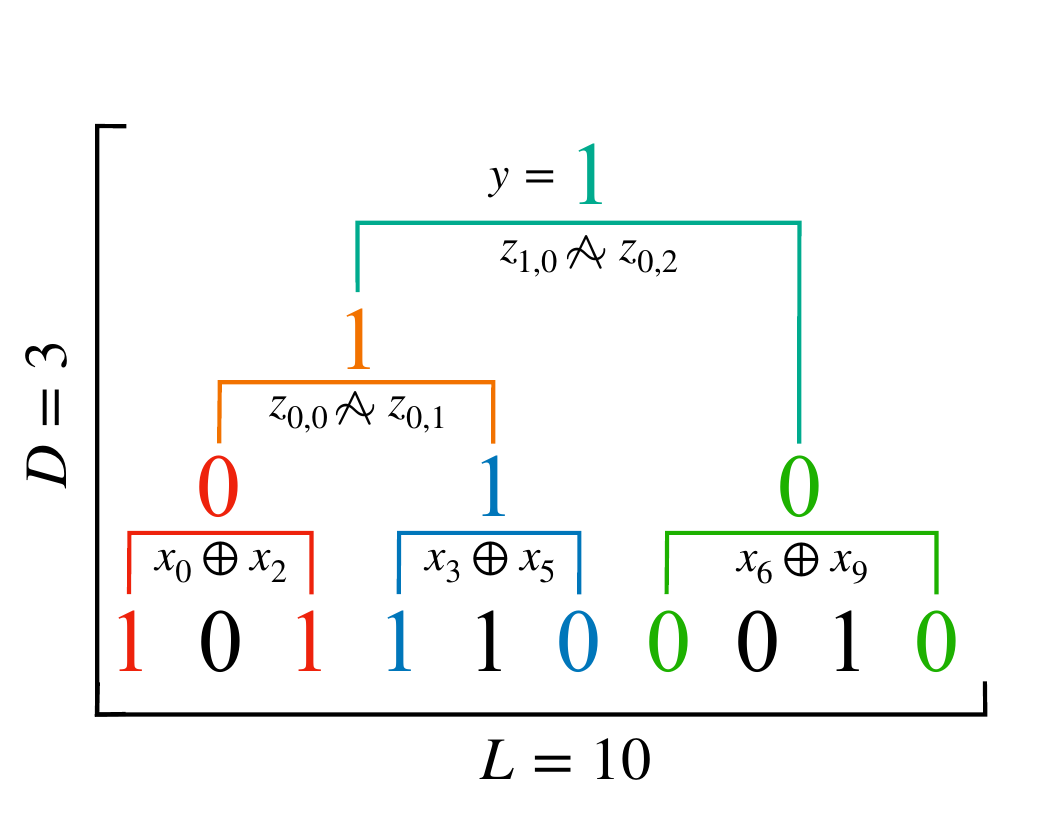}

    \captionsetup{font={small,stretch=1}}
    \caption{\textbf{HSP construction.} A length-$10$
    bit vector is read through sparse, local XORs and recursively composed by
    NAND. The learner receives only the vector, an opaque node identifier, and
    the target (Appendix~\ref{app:hsp_dynamics}).}
    \label{fig:hsp}
\end{wrapfigure}

Sparse parity (SP) is ideal for isolating feature discovery: the target depends on a small unknown subset of an otherwise unstructured input, so learning requires identifying and combining the relevant coordinates. Its Boolean factors are known exactly, so we can measure this functional progress directly rather than infer it from aggregate loss~\citep{barak2022hidden,abbe2023leap}. \citet{quanta_model} used Multitask SP as a tractable model of quantized skill acquisition, and later work considered hierarchies and compositions of these skills~\citep{michaud2025narrow}. We extend this setting with \emph{Hierarchical SP} (HSP), a controlled testbed for ordered composition.\begin{figure*}[t]
    \centering
    \includegraphics[width=0.90\textwidth]{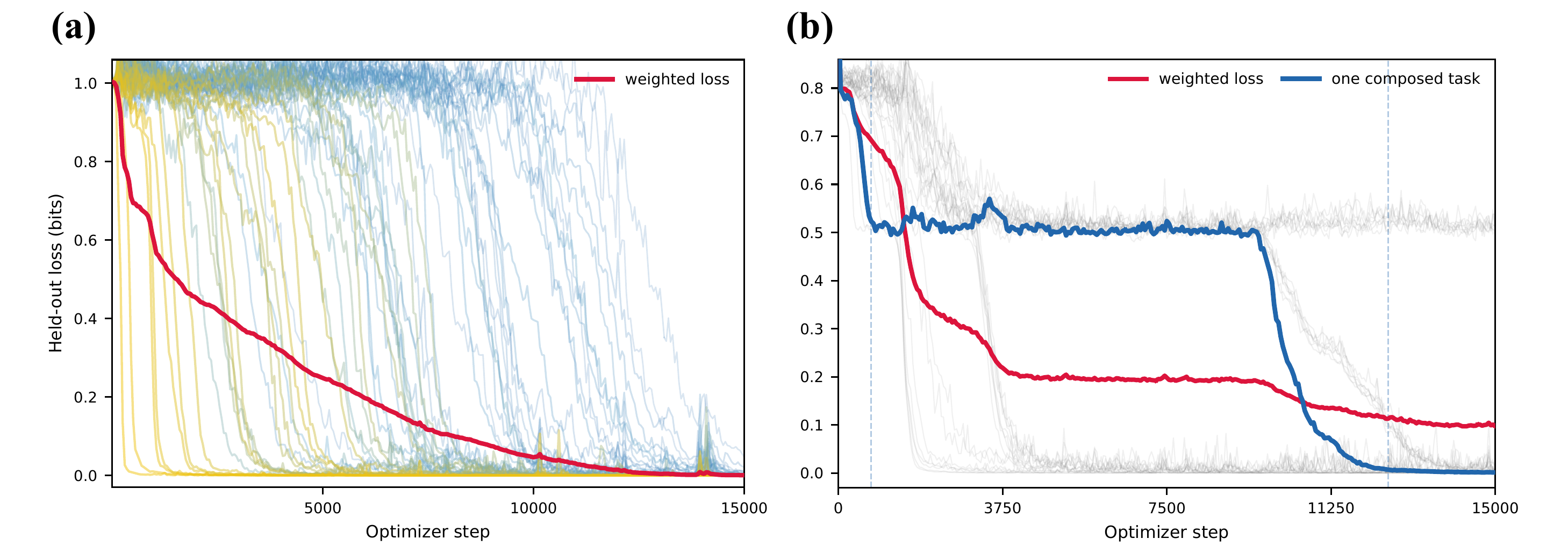}
    \caption{\textbf{Quantized loss dynamics in HSP.}
    \textbf{(a)} In a flat HSP, individual held-out losses transition sharply
    at different times, while their weighted aggregate (red) is much smoother:
    its transition is $22.6\times$ broader than the median task transition.
    Color runs from frequent (yellow) to rare (blue) tasks.
    \textbf{(b)} In a $32$-task composed HSP, one task (blue) pauses near the
    exact $1/2$-bit conditional level after acquiring its private constituent,
    then collapses when its shared constituent is acquired. Dashed lines mark
    the two $80\%$ functional-mode crossings; other task losses are gray and
    red is their weighted aggregate.}
    \label{fig:loss_overview}
\end{figure*}

HSP organizes Boolean computations as a forest (Figure~\ref{fig:hsp}). Each node $v$ owns a sparse support $S_v$ in a dense random bit vector $x$ and computes the local parity $z_v(x)=\bigoplus_{i\in S_v}x_i$. A root returns $y_v(x)=z_v(x)$, whereas every other node returns $y_v(x)=\operatorname{NAND}(y_{\operatorname{pa}(v)}(x),z_v(x))$. We use NAND because it is non-associative: moving a local parity to another depth generally changes the target, so the order of composition cannot collapse into one global parity. The learner sees no supports, tree, or intermediate outputs: HSP defines a compositional task basis without forcing the learner's execution graph.

\subsection{Loss Decomposition}
\spar{From discrete steps to smooth loss}
Let $\ell_x(t)$ be held-out cross-entropy on input type $x$. If $x$ consumes one quantum $q$, discreteness of acquisition predicts a single transition, $\ell_x(t)\simeq\ell_x^{(0)}-(\ell_x^{(0)}-\ell_x^\star)H(t-\tau_q)$. Figure~\ref{fig:loss_overview}(a) tests this prediction on $64$ degree-two parity tasks sampled at rank-dependent frequencies: their individual losses collapse at distinct acquisition times, yet their weighted mean decreases smoothly. Thus smooth overall improvement can arise from staggered functional breakthroughs rather than gradual progress on every task. If a task instead depends on multiple constituent quanta, their acquisition at different times predicts a staircase loss: learning one constituent partially solves the task and lowers its loss to an intermediate plateau, while learning the remaining constituents produces further discrete drops. Figure~\ref{fig:loss_overview}(b) realizes the simplest such case with $32$ composed tasks, each combining a private parity with one shared across tasks. The highlighted task first drops to an intermediate plateau when its private constituent is acquired, then drops again when the shared constituent is learned and the task is completed.

\spar{Implicit curriculum}
The same decomposition predicts an \emph{implicit curriculum}~\citep{liu2026implicit} even under i.i.d. sampling. When symbolic complexity is matched, demand acts as an acquisition clock: more frequent flat-HSP functions are learned earlier (Spearman $\rho=-0.80$; Figure~\ref{fig:loss_overview}(a)). With composition, demand competes with the lower conditional complexity created by an available parent. Varying only child-to-parent query frequency in a two-level HSP shifts acquisition from parent-first to predominantly child-first (Figure~\ref{fig:hsp_dynamics}), favoring respectively a deeper reuse-compatible order and a flatter direct-task route. Thus the same target functions can support different quanta poset geometries (deep or flat) under different data distributions.

\subsection{Scaling Laws from Poset Geometry}
\begin{wrapfigure}[30]{r}{0.4\textwidth}
    \centering
    \includegraphics[width=0.8\linewidth]{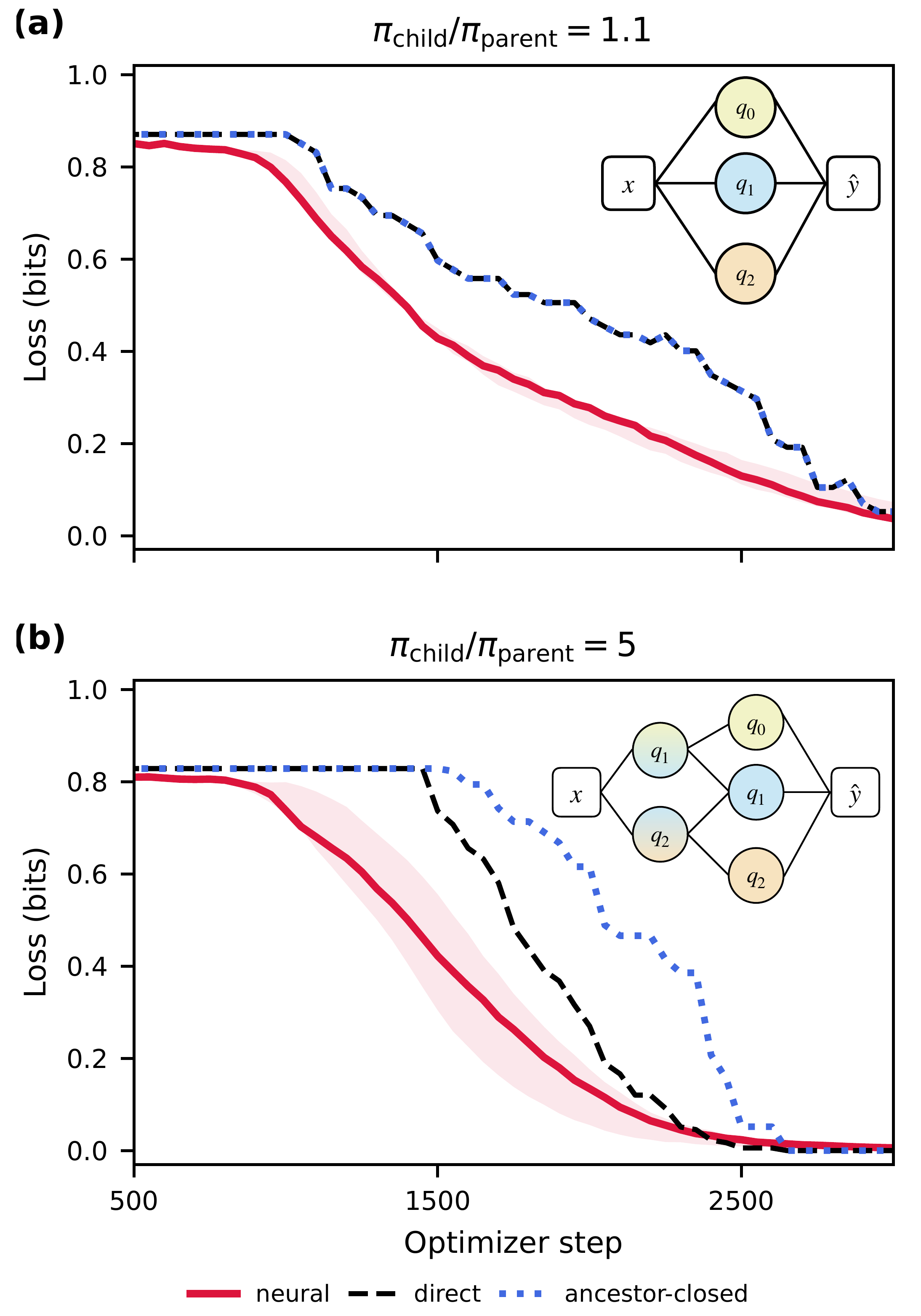}

    \caption{\textbf{Macroscopic effects of poset geometry.}
    \textbf{(a)} At ratio $1.1$ parents are acquired first and the direct-task and ancestor-closed loss reconstructions coincide, consistent with a compositional poset. \textbf{(b)} At ratio $5$, most children are acquired first and the direct-task reconstruction follows neural loss more closely, consistent with a shortcut interpretation. Curves show medians over three seeds; shading spans the range.}

    \label{fig:hsp_dynamics}
\end{wrapfigure}

\spar{From poset geometry to rank}
The experiment above suggests that the coarse geometry of a learning poset can already reveal whether a network builds on previously acquired computations or instead learns direct shortcuts. We now ask whether finer properties of this geometry can explain scaling behavior. Consider a network trained on natural language. Its learning poset is likely highly compositional: a relatively small set of broadly reused routines, such as identifying a noun phrase, lies near the root because these computations are required across many contexts and downstream operations. At greater depth, these routines can be composed into increasingly specialized skills, such as debugging an off-by-one error in a for loop. Any one such skill may occur rarely, but the number of distinct specialized computations grows rapidly with depth. The resulting geometry, therefore, induces a basic tradeoff: deeper regions of the poset contain more quanta, while each individual quantum typically receives less demand. We formalize this picture by relating the growth in the number of available computations with depth to the decay in their typical demand. Let $W(d)$ denote the cumulative number of quanta available up to depth $d$, and let $p(d)$ denote the typical demand of quanta near that depth. Since the number of quanta preceding depth $d$ determines the rank of a typical quantum at that depth, we take $k\asymp W(d)$. Over a scaling range, we assume that demand is approximately comparable within depth and that the relative growth of $W(d)$ and decay of $p(d)$ satisfy $\log p(d)=-(1+\alpha)\log W(d)+O(1)$ with $\alpha>0$. This assumption is intentionally coarse, requiring neither the number of computations nor their demand to follow a power law, but only constraining their relative rate of change so that the power law emerges from the geometry.

\spar{From rank to scaling laws}
A quantum near depth $d$ has rank $k\asymp W(d)$, so substituting rank for depth in the balanced-rate relation gives $p_k\asymp k^{-(1+\alpha)}$. If the first $K$ quanta form the learned prefix and missed quanta have comparable loss weights, summing their unmet demand yields $\mathcal L(K)-\mathcal L_\infty\asymp\sum_{k>K}p_k\asymp K^{-\alpha}$. We next map this functional prefix to ordinary resources. If the incremental representation cost of rank $k$ scales as $c_k\asymp k^{\delta_P}$, then $P(K)\asymp K^{1+\delta_P}$ and therefore $\mathcal L(P)-\mathcal L_\infty\asymp P^{-\alpha/(1+\delta_P)}$. Likewise, if rank $k$ requires $n_k\asymp k^{\delta_D}$ informative occurrences, matching these against the available $Dp_K$ examples gives $K(D)\asymp D^{1/(1+\alpha+\delta_D)}$ and hence $\mathcal L(D)-\mathcal L_\infty\asymp D^{-\alpha/(1+\alpha+\delta_D)}$ (proofs in Appendix~\ref{app:scaling_proofs}). In the experiments below, we consider $\delta_P=\delta_D=0$, so all quanta have rank-independent incremental representation and occurrence requirements.
We test these predictions by controlling HSP so that the depth--demand relation is imposed directly by the generator. We first study parameter scaling by varying the width of a fixed-depth MLP. Across the prescribed demand exponents, the resulting loss curves exhibit clear power-law regimes, with fitted slopes closely matching the theoretical predictions (Figure~\ref{fig:hsp_scaling}a). We then turn to data scaling, varying the number of online i.i.d. samples seen by the model. Here the trajectories are less clean: they begin with a shallow plateau before the first acquisition burst, pass through an intermediate learning regime, and eventually approach a second plateau; we fit only the central segment. The resulting fits again agree non-trivially with the theory but show greater trajectory instability and sensitivity to the fitting window, such as the criterion used to select its endpoint (Figure~\ref{fig:hsp_scaling}b).

\begin{figure*}[t]
    \centering
    \begin{overpic}[width=0.80\textwidth]{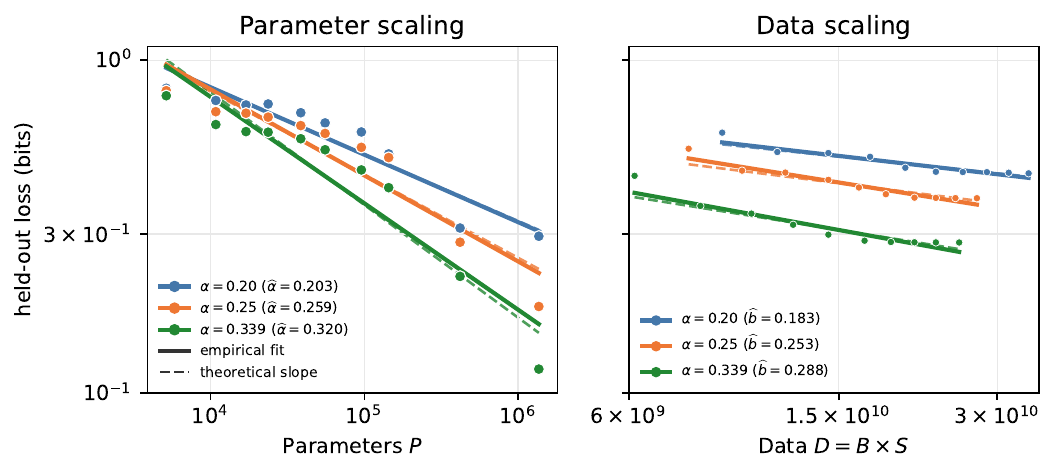}
        \put(6,42){\textbf{(a)}}
        \put(55,42){\textbf{(b)}}
    \end{overpic}
    \caption{\textbf{Fitted Scaling Laws.} \textbf{(a)} parameter scaling across ten widths. Across the three prescribed demand exponents, fitted slopes match the theoretical prediction with $1.4\%$, $3.5\%$, and $5.7\%$ relative error ($R^2\approx0.89$). \textbf{(b)} online-IID data scaling at fixed width $2048$, plotted over the central learning regime between an initial and late plateau; $D=BS$ is processed examples, with effective batch $B$ and optimizer steps $S$. Fitted slopes have $9.9\%$, $26.7\%$, and $13.7\%$ relative error from the prediction ($R^2 \approx 0.90$). Each color denotes the demand exponent $\alpha$; parenthesized values give the fitted parameter exponent $\widehat{\alpha}$ or data exponent $\widehat{b}$. Solid lines are log--log fits, and dashed lines give the prescribed slopes ($\alpha$ on the left and $\alpha/(1+\alpha)$ on the right).}
    \label{fig:hsp_scaling}
\end{figure*}

\section{Quanta Discovery}
\label{sec:number_naming}

We now recover a Q-model in a less controlled setting from training checkpoints ($\mathcal N\!\rightarrow\!\mathcal Q$) and inspect the neural trajectory in sparse coordinates. We also compile a complete functional program into a Q-model whose ordered updates guide neural training ($\mathrm{code}\!\rightarrow\!\mathcal Q\!\rightarrow\!\mathcal N$).

\subsection{Recovering a Q-Model}

\spar{Number Naming} Number naming is the problem of converting a sequence of digits into its reading as a cardinal \emph{number name} in a given language. For example, the digit sequence \textit{123} could be read as \textit{one hundred twenty three} in English. Number names involve a small vocabulary: words for basic numbers such as \textit{one}, \textit{fifteen}, and \textit{hundred}, which are not trivially derived from combinations of other numbers and thus must be learned. These atoms are combined by a limited set of syntactic rules controlled by a semantics that mostly reduces to sums of products of powers of the base~\citep{Hurford:75}. Children learn number naming from limited data: the basic words, learning to count to 100, and perhaps a few explicitly taught examples of larger numbers \cite{Sproat:22}. We use number naming because its well-understood structure makes recovered computations easy to interpret, while a short compositional procedure solves the task even with realistically limited training data on which Transformers generalize imperfectly.

\spar{Priority factors}
We train a three-layer decoder-only Transformer on 300 unique numbers spanning
one to six digits. Program-guided models are evaluated on a balanced set of $16{,}384$ numbers, including every number from 1 to 999 and sampled four- to six-digit numbers (Appendix~\ref{app:number_naming_data}), and specialized held-out panels (Appendix~\ref{app:q_discovery_details}). Our first objective is to recover an executable $\mathcal Q_t(x)$ that tracks $\mathcal N_t(x)$ along the source trajectory.
Let $e$ denote one teacher-forced prediction event and let $t$ index checkpoints. At source layer $\ell$, we measure the event's first-order credit under the complete population update as $P_\ell(e,t)\simeq\Lambda_t\langle\nabla_{\theta_\ell}L_e,
\nabla_{\theta_\ell}L_{\mathrm{train}}\rangle$, where $\Lambda_t=\sum_{s\in t}\eta_s$
is the interval learning-rate mass. Positive priority means that the population
update helps event $e$ through layer $\ell$. Thanks to \emph{quantization of
acquisition} we can decompose the priority in a small binary--temporal factorization $P_\ell(e,t)\approx\sum_q a_{\ell q}(e)\rho_{\ell q}(t)+R_\ell(e,t)$, where $a_{\ell q}(e)\in\{0,1\}$ and $\rho_{\ell q}(t)\geq0$. Starting from $R_\ell=P_\ell$, we add one factor at a time. At fixed $\rho$, an event enters the binary support exactly when doing so lowers its weighted squared residual; at fixed support, $\rho$ becomes the nonnegative support-weighted mean residual. We stop when the next factor explains less than a fraction $\tau$ of the \emph{original} host-layer squared error, leaving unexplained priority in $R_\ell$. The selected interpretability ablation uses $\tau=3\%$, giving 13 candidate coordinates. Figure~\ref{fig:qmodel_recovery}(b) shows their dynamics: three curves acquire nearly all their mass in the first saved interval, while the rest form separated bands between roughly 250 and 1,500 source steps.

\spar{Executable model}
We next freeze $a_{\ell q}$ and $\rho_{\ell q}$ and fit an executable Q-model by replaying 200 uniformly spaced source checkpoints. To approximate the Transformer's main operations, we model a block's read--write split using causal attention to read relevant prefix information and a sparse bank of scalar writers to reconstruct residual MLP updates and approximate the nonlinearity of $f_i$ in quanta (Section~\ref{sec:quantized_model}). The main loss terms are $\mathcal L_{\mathrm{rec}}+\lambda_g\mathcal L_{\mathrm{gate}}+\lambda_E\mathcal L_{\mathrm{edge}}$: $\mathcal L_{\mathrm{rec}}$ measures normalized source-block update error, $\mathcal L_{\mathrm{gate}}$ trains each local gate $g_{\ell q}$ against its frozen support $a_{\ell q}$, and $\mathcal L_{\mathrm{edge}}$ penalizes selected message channels. A hard-forward, soft-backward mask selects an acyclic set of earlier-to-later channels. With $h_q=h_0+\sum_{p\to q}m_p$, its message is $m_q=g_q(h_q)\sum_{w\in W_q}\gamma_w(h_q)\phi_w(h_q)v_w$. A root reads only $h_0$; when an edge $p\!\rightarrow\!q$ is selected, the child also receives the parent's actual message $m_p$. An edge must therefore improve reconstruction by enabling reuse, rather than merely record that two quanta occur in successive layers. An inactive parent contributes zero but does not close the child's local gate.

\subsection{Recovered Structure}

\begin{figure}[!t]
    \centering
    \includegraphics[width=0.85\textwidth]{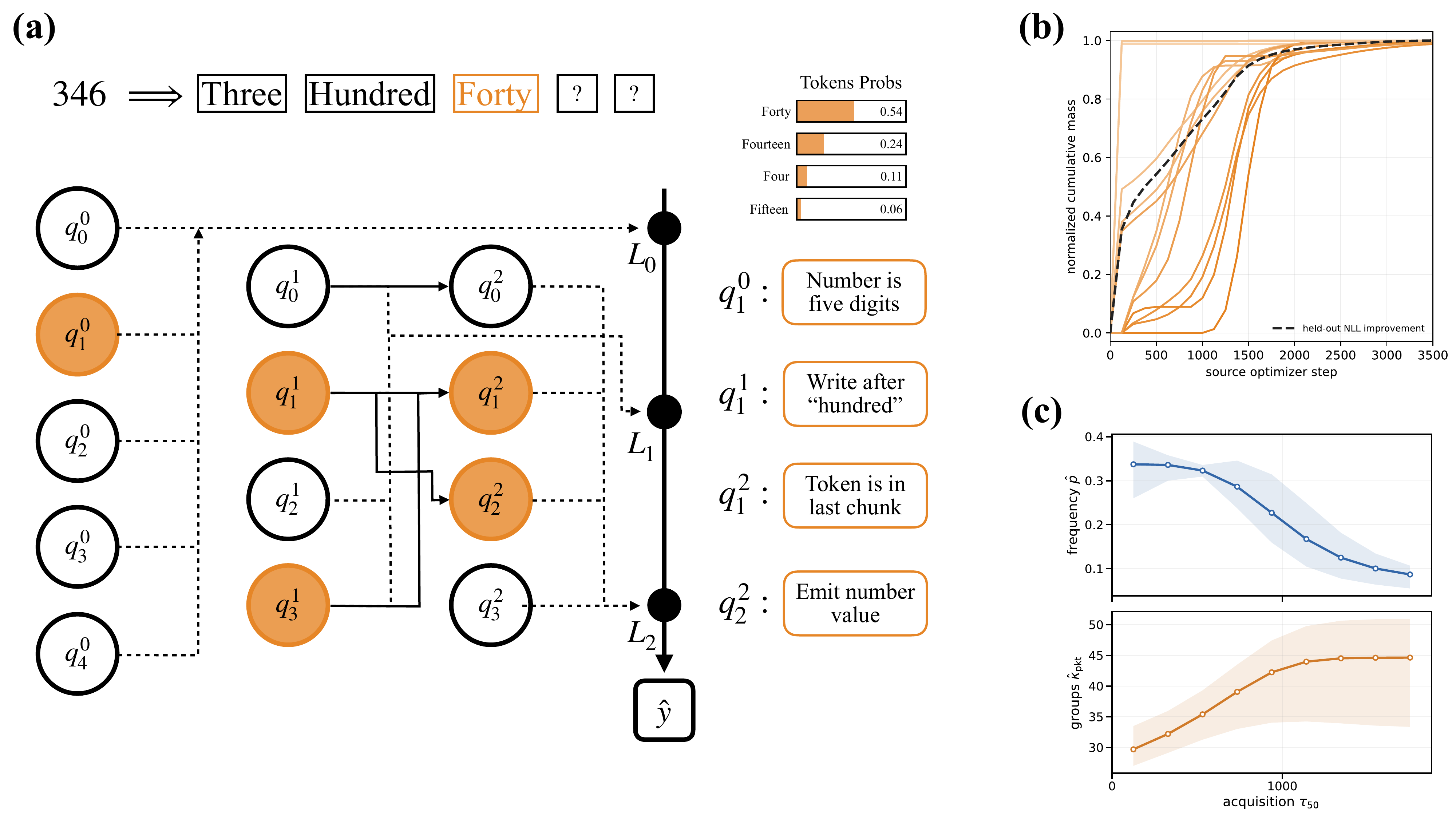}
    \caption{\textbf{From a neural trajectory to an executable
    Q-model.} \textbf{(a)} The third output step of the
    $346\rightarrow$\emph{three, hundred, forty, six, EOS} greedy trace. The source and Q-model both predict \emph{forty}; the bars show the Q-model's top-token probabilities. Orange nodes mark five active quantum gates, but only four writers emit nonzero contributions: in the current Q-model, gate eligibility and executed writer activity can disagree. Four frozen held-out probe labels are shown. Solid arrows denote selected inter-quantum message channels; dotted lines show the layerwise residual aggregation. \textbf{(b)} Normalized cumulative population-gradient priority for the 13 frozen candidates in the selected $3\%$ factorization. The dashed held-out NLL-improvement curve tracks their mean along this trajectory ($r=0.991$). \textbf{(c)} Across a finer 22-factor diagnostic basis, later factors are less frequent and their priority-weighted source updates involve more effective MLP groups.}
    \label{fig:qmodel_recovery}
\end{figure}

The recovered Q-model uses 35.7\% of the source Transformer's parameters while retaining much of its teacher-forced behavior (84.7\% source--Q argmax agreement). Its 39 available writers are sparse in execution (6.21 active on the average prediction event). The mean acquisition curve closely tracks the source model's held-out NLL improvement (Figure~\ref{fig:qmodel_recovery}b) and remains visibly step-like despite having no step-shaped training objective. Thus one sparse object describes both the trained function and when its components emerged. To test these coordinates systematically, we define a compact library of Boolean probes from the known task structure (e.g. output word and position, preceding word, value role, input digits). On a calibration split we select the best one- or two-property description for every quantum gate and writer, then freeze it and report F1 on disjoint events; writers are scored only when their owner quantum is active~\citep{cunningham2023sae}. Twenty-three of thirty-nine writers exceed held-out F1 $0.7$. Figure~\ref{fig:qmodel_recovery}a makes these statistics concrete for \textit{346}. The selected message channels matter causally within the Q-model: joint deletion or matched replacement of parent messages substantially degrades its behavior.
The acquisition order also agrees with the theory without being constrained to do so as the discovery algorithm imposes no explicit frequency or complexity constraint. In Figure~\ref{fig:qmodel_recovery}(c), we estimate demand $p_q$ by the frequency of its event support and approximate $\kappa_q$ by the effective number of source-MLP update groups carrying the factor's priority-weighted update (Appendix~\ref{app:q_discovery_details}). This packet statistic is only a simple proxy for the theoretical parent-conditioned complexity, but the predicted ordering is consistent: later factors tend to be less frequent and $p_q/\kappa_q$ decreases.

\subsection{Program-Guided Training}

The recovered Q-model motivates a second construction: probes associated with \textit{internal zeros} (e.g. \textit{3002}) activate only weakly, mirroring the source Transformer's largest generalization errors. The 300-example training set provides little evidence for this computation, raising a motivating question: can supplying the missing structure improve neural learning? A Q-model's computation graph can be read as a functional program: each quantum specifies a reusable computation, its gate determines when it executes, and each edge specifies which computation's output is supplied to another. We write a compact Python implementation of number naming whose annotated subroutines define quanta, then compile its execution conditions, data and control dependencies into a 20-quantum, 25-edge, ten-level graph (Appendix~\ref{app:qprogram_code}). We instantiate this graph as a neural Q-model with the same basic gated-message architecture as the recovered model and train it from scratch. Code traces identify the active quanta for every example, providing gate supervision throughout the graph, while messages and the readout remain learned. Its strong standalone generalization is expected given this inductive bias, but verifies that the compiled structure is executable: it reaches 97.8\% exact-sequence accuracy, compared with 84.2\% for a roughly parameter-matched Transformer.

\begin{wraptable}{r}{0.50\textwidth}
    \vspace{-0.5\baselineskip}

    \caption{\small \textbf{Q-guided training.} 
    Exact-sequence accuracy for a matched $10$-block Transformer. The zero-gap and short-tail columns contain $4{,}031$ and $1{,}553$ train-disjoint examples.}
    \label{tab:qmodel_steering}

    \vspace{-0.35em}
    \centering
    \resizebox{0.88\linewidth}{!}{%
        \begin{tabular}{lrrr}
            \hline
            Training & Balanced & Zero gap & Short tail \\
            \hline
            CE only & $84.9$ & $61.7$ & $50.5$ \\
            $+$Q target & $\mathbf{94.7}$ & $\mathbf{89.0}$ & $\mathbf{89.9}$ \\
            \hline
        \end{tabular}
    }

    \vspace{-0.55\baselineskip}
\end{wraptable}

More importantly, we use this trained Q-model as a removable target for an ordinary Transformer. For each prediction event, the frozen Q-model sums its messages at depth $d$ into $Z_d$. A fixed semi-orthogonal map aligns this target with the corresponding block update $\Delta h_d$ of a ten-layer Transformer, trained with
$\mathcal L_{\mathrm{guided}}=\mathcal L_{\mathrm{CE}}+
\lambda_Q\mathbb E_{e,d}\|U^\top\Delta h_d(e)-
\operatorname{sg}[Z_d(e)]\|_2^2$.
The Q-model is absent from evaluation and deployment. This training target improves exact-sequence accuracy from 84.9\% to 94.7\%, with the largest gains on zero gaps and short tails (Table~\ref{tab:qmodel_steering}); the Transformer blocks also align most closely with their assigned Q-depth updates, unlike under cross-entropy alone. Thus an explicit Q-description can provide useful process supervision without constraining the deployed architecture.

\section{Discussion}
\label{sec:discussion}

\spar{Summary}
We described learning as the acquisition of reusable functional units, or quanta, and derived how their demand and conditional complexity shape acquisition order. In hierarchical sparse parity, we used known computations to study discrete loss reductions, changes in learning order across data distributions, and scaling regimes linked to the geometry of composition. We then recovered candidate quanta from a Number Naming Transformer's training trajectory and built an executable Q-model to examine their behavior, semantics, and message passing. Finally, a Q-model derived from a program supplied training targets that improved a separate Transformer's generalization. Together, these steps connect a theory of functional acquisition to measured dynamics, model interpretation, and guided training.

\spar{Limitations}
Our experiments make the proposed units and their effects unusually accessible, but they do not yet test the framework on a complex, uncontrolled dataset such as natural language or on large language models. Number Naming introduces an ordinary Transformer and a recovered decomposition, yet its data still follow a compact, known procedure. Applying the method to less structured tasks will require checking whether recovered units remain stable across runs and whether interventions on the Q-model reflect mechanisms in the source model.
The loss and scaling analysis is grounded in hierarchical parity, where a model may be able to acquire each parity computation in one coordinated change. It remains unclear how far this discrete-acquisition picture extends to complex skills that may develop gradually or combine many partly learned routines. The scaling results also need broader validation across seeds, architectures, resource ranges, and data distributions, especially where finite-range fits depend on the chosen window. Our current derivations treat resource regimes separately under specific assumptions; a further theoretical goal is a single formula that combines parameter, data, and other resource scaling and predicts their crossovers.

\spar{Conclusion}
We believe the quanta abstraction can provide a common framework for studying how models acquire, reuse, and compose computations. It offers a way to relate learning dynamics, interpretable behavior, and aggregate performance across different models and tasks.

\clearpage

\bibliography{references}
\bibliographystyle{others/iclr2027_conference}

\clearpage
\appendix
\section*{\Large Appendix}
\section{HSP Construction and Additional Dynamics}
\label{app:hsp_dynamics}

\subsection{Generator and Task Construction}
HSP begins with a forest of $R$ roots, branching factor $\rho$, and maximum depth $D_{\max}$. Every node is assigned a unique degree-two support sampled from a common pool of sensor coordinates. We reject a proposed support if it has already appeared elsewhere or if its parity mask is linearly dependent over $\mathrm{GF}(2)$ on the masks earlier in the same root-to-node path. Consequently every level introduces new information, while supports on different branches may still overlap through the shared pool. Recursive NAND is non-associative, so moving one of these local parities to another depth generally changes the target function.

To produce an example, we sample independent balanced sensor bits $x$, choose one queried node $v$ with probability $\pi_v$, and return only $(x,v,y_v(x))$. The learner receives the complete sensor vector and an opaque task identifier, but not the support sets, forest edges, local parities, or ancestor outputs. The hierarchy is therefore a generator-side compositional basis rather than an enforced neural execution graph. For the scaling experiments, we first prescribe ancestor-closure demands by depth and then recover the nonnegative direct query probabilities by tree M\"obius inversion; Appendix~\ref{app:hsp_sampler} gives the exact construction and its relation to the target exponent.

\subsection{Conditional-Loss Staircases}
For a composed task, let $B_x(S)$ be the best loss attainable when only $S\subseteq A(x)$ is available. The corresponding Q-prediction is the finite staircase $\ell_x(t)\simeq B_x(I_t\cap A(x))$: each newly available constituent removes its conditional part of the loss. These levels can be computed exactly in small Boolean compositions. For a NAND of two independent balanced inputs, the Bayes cross-entropy is $H_2(3/4)$ bits with neither input known, $1/2$ bit with one known, and zero with both. These are predictions from a declared functional basis; observing the levels alone does not identify the model's internal representation.

Figure~\ref{fig:loss_overview}(b) tests the two-factor prediction with $32$ task-conditioned targets. Each target NANDs a task-private degree-one parity with a degree-three parity shared by eight tasks. The two roles draw unique supports from separate fixed $32$-bit pools, with another $32$ dense distractor bits. A two-hidden-layer width-$256$ ReLU MLP is trained with Adam on fresh i.i.d. batches of $512$ examples and evaluated every $50$ steps on an exact balanced panel. For the highlighted task, the private mode crosses $80\%$ at step $750$ and the shared mode at step $12{,}550$. Its held-out loss stays near the predicted $1/2$-bit conditional plateau between the two and first reaches exact classification at step $11{,}400$. The run acquires all $36$ declared modes, and all $32$ tasks are exact simultaneously at the best checkpoint; $28/32$ remain exact at the final checkpoint, where the weighted loss is $0.00264$ bits. This single run demonstrates a learned staircase in one composed task; it does not establish that the MLP internally executes the declared factorization.

\subsection{Flat Acquisition Control}
The main text uses the flat HSP run to isolate sharp primitive acquisitions. Figure~\ref{fig:loss_overview}(a) contains $64$ unique degree-two parities sampled from a shared $32$-bit pool, an additional $32$ dense distractors, and task frequencies proportional to $r^{-1}$. A two-hidden-layer width-$256$ MLP is trained with Adam at batch size $512$. The displayed run acquires all $64$ functional coefficients by step $13{,}700$ and all tasks exactly by step $13{,}500$; frequency and acquisition time have Spearman correlation $-0.800$ and the log--log clock fit has $R^2=0.810$. With the frequency assignment fixed, two fresh model seeds acquire $63$ and $62$ coefficients and give correlations $-0.766$ and $-0.709$ by step $20{,}000$. Two fresh rank assignments acquire $64$ and $62$ coefficients and give correlations $-0.712$ and $-0.629$. Thus the frequency curriculum is replicated but imperfect: task-specific optimization and shared-feature geometry still contribute to timing.

\section{HSP Acquisition Diagnostics}
\label{app:hsp_diagnostics}

\spar{Acquisition criterion and two-level ordering}
HSP acquisition times are measured with a bias-insensitive held-out functional coefficient. A target is counted as acquired only when this coefficient remains above $0.8$ for five evaluations; this avoids treating the biased marginal of a NAND target as functional learning. In the two-level comparison of Figure~\ref{fig:hsp_dynamics}, consisting of roots and one child layer, every one of the $48$ functions crosses this criterion in each of three model seeds. The pooled child-first counts are $0/96$, $7/96$, and $73/96$ at child/parent direct-frequency ratios $1.1$, $2$, and $5$. At $5\times$, the median reconstruction error through step $4{,}000$ is $0.070$ bits for the direct-task basis and $0.136$ bits for the ancestor-closed basis; at $1.1\times$ the two reconstructions coincide.

\spar{Acquisition clock across depth}
The separate $60$-node hierarchy tests the acquisition clock over more depths. A fit of $T\propto\pi^{-\gamma}$ gives $\gamma=0.982$ and $R^2=0.914$. Adding depth as a conditional-complexity factor estimates a $1.14\times$ time cost per level and raises $R^2$ to $0.932$. This supports a soft depth cost after accounting for direct frequency, not a strict prerequisite relation or internal reuse of parent messages.

\section{The HSP Demand Sampler}
\label{app:hsp_sampler}

\spar{From closure demand to query probability}
For the scaling experiment, consider $R$ regular $\rho$-ary trees and prescribe the per-node ancestor-closure demand $P_v=R^{-1}\beta^{-d(v)}$.

The learner is trained on one terminal query per example, so these overlapping closure demands must be converted into a distribution over queried nodes. Tree M\"obius inversion gives $\pi_v=P_v-\sum_{c\in\operatorname{ch}(v)}P_c$. For non-leaves this simplifies to $\pi_v=P_v(1-\rho/\beta)$, while $\pi_v=P_v$ at the leaves. The condition $\beta>\rho$ makes every probability nonnegative. Summing this identity over the forest telescopes to the root mass, $\sum_v\pi_v=\sum_{r\text{ root}}P_r=1$. Moreover, the probability that a query lies in the subtree of $v$ is exactly $P_v$, so sampling $v\sim\pi$ and evaluating its ancestor path realizes the prescribed closure marginals.

\spar{Finite-depth effects}
The finite boundary is visible. The total terminal mass at depth $d<D_{\max}$ is $(\rho/\beta)^d(1-\rho/\beta)$, while leaf mass is $(\rho/\beta)^{D_{\max}}$. For the present $\rho=2$, $\beta=2^{1.2}$, and $D_{\max}=5$, one half of all terminal queries are leaves. This makes direct leaf shortcuts more attractive than in an infinite hierarchy and is one reason the model-side frontier must be measured separately from the generator-side demand law.

\section{HSP Scaling Experiments}
\label{app:hsp_scaling_experiments}

\subsection{Parameter-Sweep Protocol}
The parameter sweeps use $R=8$, branching factor $\rho=2$, maximum depth $5$, and $504$ task identities. The graph construction, sparse supports, model, data and evaluation seeds, optimizer, and evaluation panel are fixed across widths $8,16,24,32,48,64,96,128,256,512$. Each learner is a five-hidden-layer task-conditioned ReLU MLP trained with maximal-update parameterization (base/delta widths $8/16$) and MuSGD at learning rate $0.0075$ and zero weight decay. The effective batch and microbatch sizes are both $16{,}384$, so every update is computed in one pass. The support, rank, and model seeds are $0$, the online-data seed is $7000$, and the evaluation seed is $10000$. The three values $\beta=2.29739671$, $2.37841423$, and $2.53$ prescribe $\alpha=0.20$, $0.25$, and $0.339137$, respectively; all other fields are matched. At each evaluation, the same $256$ sensor draws are reused for every task and checkpoint. We compute taskwise held-out binary cross-entropy in bits and weight it by the direct query probabilities $\pi_v$ before fitting.

\subsection{Parameter Fits and Stability}
We fit $\log \mathcal L=a-\widehat\alpha\log P$ by ordinary least squares over all ten widths, where $P$ is total trainable parameters, including task conditioning, and $\mathcal L$ is the weighted held-out loss without an asymptote subtraction. We report $R^2$ in the same log--log space and relative slope error $|\widehat\alpha-\alpha|/\alpha$. To reduce evaluation noise, every plotted point is the median of the evaluations in the $25{,}000$ steps ending at its selected checkpoint. For $\alpha=0.20$ and $0.25$, all widths use the common $15$M and $10$M checkpoints. For $\alpha=0.339$, stable widths use $20$M; width $24$ uses $18$M after collapsing by $19$M, and widths $128$ and $256$ use $19$M before late instabilities at $20$M. Thus the latter curve uses the last stable integer-million checkpoint in a fixed $18$--$20$M window, rather than selecting checkpoints by proximity of the fitted slope to its target. This stability-window rule is an exploratory response to observed late failures and should be fixed in advance for replication.

The fitted exponents are $0.20285$, $0.25870$, and $0.31996$, giving relative errors $1.43\%$, $3.48\%$, and $5.66\%$, with $R^2=0.8951$, $0.8863$, and $0.8913$. The approach to these slopes is not a single-checkpoint accident: representative late common fits retain high log--log agreement ($R^2=0.895$ at $15$M for $\alpha=0.20$, $0.892$ at $9$M and $0.886$ at $10$M for $\alpha=0.25$, and $0.904$, $0.899$, and $0.898$ at $15$M, $17$M, and $18$M for $\alpha=0.339$). The specified checkpoint fits are close to the prescribed slopes, but we observe sporadic late optimization instabilities, possibly due to muP, in which models temporarily forget acquired tasks before recovering with further training to roughly their previous loss and acquisition count.

\subsection{Data-Scaling Protocol}
For the data panel, $D=BS$ denotes processed fresh i.i.d. examples, where $B=4096$ is the effective batch and $S$ is the optimizer-step count; $D$ is not the size of a stored training set. We use one width-2048 trajectory per exponent, with tree depth eight ($4088$ tasks), eight hidden layers, and a fixed $256$-sensor evaluation panel. Seeds and optimization are matched across exponents. The $\alpha=0.339137$ trajectory resumes from its 8M-step checkpoint. At each 100k-step evaluation, we take the median weighted held-out loss over that point and its immediate neighbors. A distinct acquisition event is an increase in the number of functional coefficients above $0.8$; the fit starts at the third such event. A new best loss must improve by at least $0.05\%$ to reset a ten-evaluation patience counter. The endpoint is the best checkpoint immediately before that counter expires, and ordinary least squares fits all smoothed losses in the selected interval as $\log\mathcal L=a-\widehat b\log D$, without subtracting a loss floor. The resulting windows are 2.2M--8.4M, 1.9M--6.7M, and 1.5M--6.2M steps for $\alpha=0.20$, $0.25$, and $0.339137$, respectively. The dashed data references use slope $\alpha/(1+\alpha)$ and a re-centered intercept. This shared post-hoc convention avoids theory-targeted window selection, but it is not a convergence estimator.

\subsection{Fit-Window Sensitivity}
To expose the finite-window dependence, we split each selected interval at the geometric mean of its endpoints and refit the earlier and later checkpoints separately, with no change to smoothing or loss definition. The table gives the fitted $\widehat b$; the target is $\alpha/(1+\alpha)$. The later slope is smaller in every arm, although each half spans only about a factor of two in exposure. A high full-window $R^2$ therefore does not establish a stationary exponent.

\begin{table}[t]
    \centering
    \caption{\small \textbf{Data-fit window sensitivity.} Equal-log-span halves of the acquisition/plateau-selected interval; ranges are optimizer steps and every fit uses $D=4096S$.}
    \label{tab:hsp_data_window_sensitivity}
    \small
    \setlength{\tabcolsep}{4pt}
    \begin{tabular}{rrrll}
        \hline
        $\alpha$ & Target & Full fit & Earlier range / slope & Later range / slope \\
        \hline
        $0.20$ & $0.1667$ & $0.1832$ & $2.2$--$4.2$M / $0.1882$ & $4.3$--$8.4$M / $0.1072$ \\
        $0.25$ & $0.2000$ & $0.2534$ & $1.9$--$3.5$M / $0.2838$ & $3.6$--$6.7$M / $0.1532$ \\
        $0.339$ & $0.2533$ & $0.2879$ & $1.5$--$3.0$M / $0.4339$ & $3.1$--$6.2$M / $0.1164$ \\
        \hline
    \end{tabular}
\end{table}

\section{Derivations for the Scaling Regimes}
\label{app:scaling_proofs}

\subsection{Demand-Ranked Loss}
Throughout, $a_k\asymp b_k$ means that $c b_k\leq a_k\leq C b_k$ for constants $0<c<C<\infty$ independent of rank and resource scale over the stated range. The asymptotics describe a countable family, or a sequence of finite systems before saturation. A fixed finite HSP eventually exhausts its quantum vocabulary and leaves this regime.

\begin{lemma}[Power-law tail]
\label{lem:power_law_tail}
For $\alpha>0$, $\sum_{k>K}k^{-(1+\alpha)}\asymp K^{-\alpha}$. Consequently, if $p_k\asymp k^{-(1+\alpha)}$, then $\sum_{k>K}p_k\asymp K^{-\alpha}$.
\end{lemma}

\begin{proof}
Monotonicity and integral comparison give $(K+1)^{-\alpha}/\alpha\leq\sum_{k>K}k^{-(1+\alpha)}\leq K^{-\alpha}/\alpha$.

Multiplying these bounds by the constants implicit in $p_k\asymp k^{-(1+\alpha)}$ proves the second statement.
\end{proof}

\begin{proof}[Demand-ranked loss tail]
The scaling-range assumption is equivalent to $p(d)\asymp W(d)^{-(1+\alpha)}$. Within-depth comparability gives $p_q\asymp p(d)$, and rank regularity gives $k\asymp W(d)$ for a quantum at depth $d$. Therefore $p_k\asymp p(d)\asymp W(d)^{-(1+\alpha)}\asymp k^{-(1+\alpha)}$.

Lemma~\ref{lem:power_law_tail} and the assumed unmet-demand loss bridge then prove the demand-ranked loss relation used in Section~\ref{sec:loss_dynamics}. For the geometric specialization, $W(d)\asymp\rho^d$ and $p(d)\asymp\beta^{-d}$ give $1+\alpha=\log_\rho\beta$. For polynomial growth $W(d)\asymp d^u$ and decay $p(d)\asymp d^{-v}$, the same argument gives $1+\alpha=v/u$. Thus the result depends on their relative growth, not on a particular tree geometry.
\end{proof}

\subsection{Resource Scaling}
\begin{proof}[Resource scaling laws]
If the incremental representation cost is $c_k\asymp k^{\delta_P}$ and costs are additive up to constant factors, then $P(K)\asymp\sum_{k\leq K}c_k\asymp K^{1+\delta_P}$ and $K(P)\asymp P^{1/(1+\delta_P)}$.

Substitution into the unmet-demand tail yields $\mathcal L(P)-\mathcal L_\infty\asymp K(P)^{-\alpha}\asymp P^{-\alpha/(1+\delta_P)}$.

For data scaling, a dataset of size $D$ contains order $Dp_k$ informative occurrences of rank $k$. The sample-threshold assumption $n_k\asymp k^{\delta_D}$ therefore sets the frontier by $DK^{-(1+\alpha)}\asymp K^{\delta_D}$, hence $K(D)\asymp D^{1/(1+\alpha+\delta_D)}$.

Applying the same tail gives $\mathcal L(D)-\mathcal L_\infty\asymp K(D)^{-\alpha}\asymp D^{-\alpha/(1+\alpha+\delta_D)}$.

Together, these two substitutions give the parameter and data scaling laws used in Section~\ref{sec:loss_dynamics}.
\end{proof}

\spar{Assumptions and scope}
The resource derivation assumes a realizable ordered frontier and sufficient optimization; it does not follow merely by counting nominal parameters. The sample-threshold law, the frontier condition, and the unmet-demand loss bridge are also separate empirical assumptions. In particular, a parameter- or data-scaling curve alone cannot establish the poset mechanism.

\section{Number Naming Data and Evaluation}
\label{app:number_naming_data}

\subsection{Task Representation and Metrics}
For an integer such as \textit{42017}, the model receives the causal prompt \texttt{[BOS] <D4> <D2> <D0> <D1> <D7> [SEP]} and predicts \emph{forty two thousand seventeen [EOS]}. Decimal inputs are not zero padded. Cross-entropy begins at the first output word and includes termination; the vocabulary consists of the special symbols, ten digit tokens, and the English words present in the constructed data. We report teacher-forced token accuracy and greedy exact-sequence accuracy, with the latter requiring correct termination.

\subsection{Compositional-Generalization Split}
The \emph{compositional-generalization split} contains $300$ unique training numbers below one million. A capacity-aware allocation gives $9$ one-digit examples and $59, 58, 58, 58, 58$ examples at lengths two through six. Before filling the remaining slots, the builder reserves lexical forms, teens, exact tens, representative hundreds, internal-zero cases, thousand boundaries, exact teen-thousand blocks, and selected high/low chunk combinations. This prevents a missing primitive from masquerading as failed composition while leaving most larger combinations unseen.

The corresponding balanced evaluation, used for the program-derived and Q-guided models, has $16{,}384$ examples: every integer from $1$ through $999$ appears once, and the remaining budget is allocated without replacement across four-, five-, and six-digit numbers. The first block provides complete lexical and short-composition coverage rather than a strictly held-out set; the present evaluation contains $160$ train/evaluation overlaps. We therefore also report train-disjoint results and treat the larger-number region as the main recombination test.

\spar{Targeted diagnostics}
The zero-gap diagnostic contains the $4{,}031$ train-disjoint examples whose decimal representation includes the pattern nonzero--one-or-more-zeros--nonzero. The short-tail diagnostic contains the $1{,}553$ train-disjoint numbers at least $1{,}000$ whose final three-digit chunk lies in $1$--$99$. The 300-example training set contains 64 zero-gap numbers and 30 short-tail numbers, with the former distributed across different zero positions. These subsets isolate optional-place suppression and cross-thousands recombination rather than being selected by model correctness.

\subsection{Length-Balanced Coverage Control}
The \emph{length-balanced split} samples an equal quota from every available digit length for training and constructs evaluation independently with the same length balance. Sampling is with replacement and train/evaluation overlap is allowed. This regime is not the strongest generalization test; it is a coverage control that prevents the much larger set of long numbers from dominating aggregate metrics and priority estimation. Digit-wise evaluation reports each length separately before averaging.

\section{Quanta Discovery, Compilation, and Steering Details}
\label{app:q_discovery_details}

The functional decomposition does not require the corresponding parameters to be disentangled: several quanta may be superposed or share parameters, provided that their contributions can still be tracked at the chosen functional resolution.

\subsection{Discovery and Executable Fitting}
\spar{Priority factorization}
The reference source is a three-layer, width-$32$, one-head Transformer trained
for $5{,}000$ exact full-population gradient steps on the $300$ examples.
Q-discovery measures the population-gradient priority defined in
Section~\ref{sec:number_naming} from the complete update. Within each layer,
we greedily fit the binary--temporal model
$P_{et}\approx\sum_q a_q(e)\rho_q(t)+R_{et}$. For current residual $R$, one
component minimizes $\sum_{e,t}w_e(R_{et}-a_e\rho_t)^2$, with uniform event
weights $w_e$. At fixed $\rho$, $a_e=1$ exactly when
$2\sum_tR_{et}\rho_t>\sum_t\rho_t^2$; ties are excluded. At fixed support,
$\rho_t=\max(0,\sum_ew_ea_eR_{et}/\sum_ew_ea_e)$. Empty supports are rejected,
and alternating updates stop after unchanged support and curve or 40 steps.
The curves remain in interval-priority units; only their cumulative positive
mass is normalized for the timing visualization.

Candidate fitting starts from four
temporal initializations---the positive weighted residual mean, the positive
and negative leading right-singular vectors clipped to the nonnegative
orthant, and a seeded random curve. Each initialization receives at most 40
alternating support/curve updates, and the lowest-error candidate is proposed.
The residual is updated only if this candidate clears the stated fraction of
the layer's original priority SSE. Consequently, the component ceiling does
not determine the retained count, and changing the threshold refactorizes the
saved priority field rather than merely changing an evaluation setting.

The interpretability model is selected from a controlled sweep
that varies this threshold, writer width, and one global per-event writer
target while keeping the source trajectory, replay, seed, and objectives
fixed. Its $3\%$ threshold retains 13 candidates, distributed $[5,4,4]$
across source layers; every retained candidate clears the $3\%$ test against its host layer's original error. The $1\%$, $3\%$, and $5\%$ refactorizations are threshold-robustness diagnostics of the same saved priority field.

On the held-out priority panel, the selected binary model explains $52.6\%$,
$52.7\%$, and $44.9\%$ of SSE across the three layers. A rank-matched signed
SVD baseline explains $79.0\%$, $74.7\%$, and $78.1\%$; negative priority
accounts for $19.0\%$, $20.6\%$, and $20.0\%$ of absolute priority mass. This
baseline tests the reconstruction cost of the binary, nonnegative restriction,
not an alternative semantic Q-model.

\spar{Executable fitting}
The Q-model replays $200$ uniformly spaced source checkpoints in IID order for
$5{,}000$ persistent Adam updates. At every replayed checkpoint, the frozen
factor curves set candidate existence and the frozen supports label the local
gates. The functional loss reconstructs source block updates at all valid token
positions, normalized by each source layer's total update energy; output KL is
evaluated at prediction events. Gate
BCE is class-balanced, and reconstruction and KL gradients are detached from
the theory-facing gate logits. Writer and edge-count penalties are ramped in
over the first $1{,}000$ Q updates. Channel entropy begins after update $2{,}500$
and the final hard graph is frozen before evaluation. Writers use fixed-threshold
JumpReLU gates and unit-normalized output directions; no dense attention-output
path bypasses the writer dictionary.

The selected Q-model has $10{,}432$ parameters versus $29{,}196$ in the source,
three writers per candidate, 39 available writers, and four selected global
message channels. A target of six total active writers per event realizes
6.21, alongside 4.66 active quanta per event.

\subsection{Evaluation and Semantic Diagnostics}
On the saved $300$-example held-out panel, source and Q-model teacher-forced token accuracies are $93.1\%$ and $86.3\%$, source--Q token agreement is $84.7\%$, and their greedy exact-sequence accuracies are $70.3\%$ and $45.7\%$. Semantic descriptions are selected on a separate $2{,}048$-example calibration panel and evaluated unchanged on a disjoint $2{,}048$-example test panel. On the latter, Q-model token accuracy is $84.7\%$ and source--Q agreement is $83.3\%$; median rich held-out F1 is $0.698$ for quanta, with $6/13$ at least $0.7$, and $0.719$ for writers conditional on their owner quantum, with $23/39$ at least $0.7$. On a separate $512$-example causal panel, joint deletion of the four messages lowers accuracy by 8.0 points and matched-message replacement lowers it by 2.5 points.

The audit finds clean tail, hundred, and preceding-token writers but no reusable
zero writer: the best single-writer held-out F1 is $0.12$--$0.25$ across place
values. This motivates the zero-gap and short-tail diagnostics, but does not
show that the source Transformer lacks a distributed zero-handling mechanism.

\subsection{Execution Traces and Ordering Diagnostics}
\spar{Exact $346$ trace}
The source and Q-model both greedily generate \emph{three, hundred, forty, six, EOS}. Tables~\ref{tab:qmodel_346_summary}--\ref{tab:qmodel_346_writers} report every active quantum and writer in the trace. Activity is the magnitude of a gated scalar-writer contribution, not a probability or causal-effect size; F1 is measured on the held-out semantic test. The descriptions follow the coherent progression sequence start and value read $\rightarrow$ hundreds structure $\rightarrow$ after-\emph{hundred} and tail processing $\rightarrow$ unit value $\rightarrow$ termination. Across the 26 active-writer occurrences, mean owner-conditional held-out F1 is $0.802$ and 23 occurrences come from writers with F1 at least $0.7$.

\begin{table*}[t]
    \centering
    \caption{\small \textbf{Event summary for the $346$ greedy trace.} Counts and activity are exact for the displayed execution; writer F1 statistics use the frozen owner-conditional descriptions.}
    \label{tab:qmodel_346_summary}
    \small
    \setlength{\tabcolsep}{5pt}
    \begin{tabular}{lrrrrl}
        \hline
        Output & Active & Active & Total & Mean/min  & Event interpretation \\
               & Qs     & writers & activity &      writer F1   & \\
        \hline
        three & 4 & 6 & 39.600 & 0.812 / 0.741 & sequence start, unit lexeme, value read \\
        hundred & 3 & 4 & 12.088 & 0.696 / 0.584 & hundreds and low-chunk structure \\
        forty & 5 & 6 & 33.803 & 0.847 / 0.771 & after-hundred, tail, value read/write \\
        six & 4 & 7 & 20.187 & 0.821 / 0.650 & unit value and tail completion \\
        EOS & 7 & 3 & 6.095 & 0.794 / 0.684 & termination and completed tail \\
        \hline
    \end{tabular}
\end{table*}

\begin{table*}[t]
    \centering
    \caption{\small \textbf{All active quanta for $346$.} Each description is selected on calibration events and its F1 is evaluated unchanged on the disjoint semantic test.}
    \label{tab:qmodel_346_quanta}
    \scriptsize
    \setlength{\tabcolsep}{5pt}
    \begin{tabular}{llrl}
        \hline
        Output & Quantum & Held-out F1 & Selected description \\
        \hline
        three & \texttt{L0.Q3} & 0.615 & unit lexeme or 5-digit input \\
        three & \texttt{L1.Q2} & 0.705 & after sequence start or 3-digit input \\
        three & \texttt{L1.Q3} & 0.544 & chunk has hundreds \\
        three & \texttt{L2.Q2} & 0.662 & unit lexeme \\
        hundred & \texttt{L0.Q0} & 0.698 & chunk has tail \\
        hundred & \texttt{L1.Q0} & 0.760 & one-hundred rule or EOS decision \\
        hundred & \texttt{L2.Q0} & 0.743 & chunk has tail \\
        forty & \texttt{L0.Q1} & 0.566 & emit value or 5-digit input \\
        forty & \texttt{L1.Q1} & 0.713 & emit value and chunk has hundreds \\
        forty & \texttt{L1.Q3} & 0.544 & chunk has hundreds \\
        forty & \texttt{L2.Q1} & 0.792 & emit value and chunk has hundreds \\
        forty & \texttt{L2.Q2} & 0.662 & unit lexeme \\
        six & \texttt{L0.Q3} & 0.615 & unit lexeme or 5-digit input \\
        six & \texttt{L1.Q2} & 0.705 & after sequence start or 3-digit input \\
        six & \texttt{L2.Q1} & 0.792 & emit value and chunk has hundreds \\
        six & \texttt{L2.Q2} & 0.662 & unit lexeme \\
        EOS & \texttt{L0.Q0} & 0.698 & chunk has tail \\
        EOS & \texttt{L0.Q2} & 0.579 & EOS decision or 6-digit input \\
        EOS & \texttt{L1.Q0} & 0.760 & one-hundred rule or EOS decision \\
        EOS & \texttt{L1.Q2} & 0.705 & after sequence start or 3-digit input \\
        EOS & \texttt{L1.Q3} & 0.544 & chunk has hundreds \\
        EOS & \texttt{L2.Q0} & 0.743 & chunk has tail \\
        EOS & \texttt{L2.Q3} & 0.711 & EOS decision \\
        \hline
    \end{tabular}
\end{table*}

\begin{table*}[t]
    \centering
    \caption{\small \textbf{All active writers for $346$.} Activity is the gated contribution magnitude in this execution. F1 is held out and conditioned only on events where the writer's owner quantum is active.}
    \label{tab:qmodel_346_writers}
    \scriptsize
    \setlength{\tabcolsep}{5pt}
    \begin{tabular}{llrrl}
        \hline
        Output & Writer & Activity & Conditional F1 & Selected description \\
        \hline
        three & \texttt{L0.Q3.W10} & 0.218 & 0.741 & after sequence start \\
        three & \texttt{L1.Q2.W8} & 1.802 & 0.824 & after sequence start or hundreds \\
        three & \texttt{L1.Q3.W9} & 0.359 & 0.783 & unit form or after sequence start \\
        three & \texttt{L2.Q2.W6} & 14.219 & 0.882 & value read \\
        three & \texttt{L2.Q2.W7} & 10.736 & 0.869 & value read \\
        three & \texttt{L2.Q2.W8} & 12.266 & 0.771 & emit value \\
        hundred & \texttt{L0.Q0.W2} & 1.395 & 0.584 & after three or after four \\
        hundred & \texttt{L1.Q0.W0} & 1.772 & 0.728 & low chunk \\
        hundred & \texttt{L2.Q0.W1} & 6.231 & 0.753 & chunk has hundreds \\
        hundred & \texttt{L2.Q0.W2} & 2.689 & 0.719 & low chunk or 4-digit input \\
        forty & \texttt{L1.Q1.W3} & 0.106 & 0.828 & after hundred \\
        forty & \texttt{L2.Q1.W3} & 6.971 & 0.847 & chunk has tail \\
        forty & \texttt{L2.Q1.W5} & 9.314 & 0.883 & chunk has tail \\
        forty & \texttt{L2.Q2.W6} & 5.467 & 0.882 & value read \\
        forty & \texttt{L2.Q2.W7} & 5.906 & 0.869 & value read \\
        forty & \texttt{L2.Q2.W8} & 6.039 & 0.771 & emit value \\
        six & \texttt{L0.Q3.W9} & 1.013 & 0.650 & chunk has hundreds \\
        six & \texttt{L2.Q1.W3} & 4.822 & 0.847 & chunk has tail \\
        six & \texttt{L2.Q1.W4} & 1.294 & 0.841 & chunk has tail \\
        six & \texttt{L2.Q1.W5} & 1.933 & 0.883 & chunk has tail \\
        six & \texttt{L2.Q2.W6} & 2.247 & 0.882 & value read \\
        six & \texttt{L2.Q2.W7} & 3.845 & 0.869 & value read \\
        six & \texttt{L2.Q2.W8} & 5.033 & 0.771 & emit value \\
        EOS & \texttt{L0.Q2.W7} & 1.320 & 0.710 & one-hundred rule or 3-digit input \\
        EOS & \texttt{L1.Q3.W10} & 2.151 & 0.684 & chunk has tail \\
        EOS & \texttt{L2.Q3.W11} & 2.624 & 0.989 & chunk has tail \\
        \hline
    \end{tabular}
\end{table*}

\spar{A contrasting execution}
Figure~\ref{fig:qmodel_800_trace} shows a second saved greedy trace from the same 13-quantum, 39-writer Q-model. On input $800$, both the source and Q-model emit \emph{eight, hundred, EOS}, so the model terminates without a tens or unit word. Four, three, and five gates are active on the three events, but only four, four, and three writers emit. In particular, an active gate can contribute no writer message; gate activity alone is not a measure of computation. This example was selected after fitting from a ranked probe set and is an illustration, not a held-out frequency estimate or evidence that the source uses these writer identities.

\begin{figure}[t]
    \centering
    \includegraphics[width=0.60\textwidth]{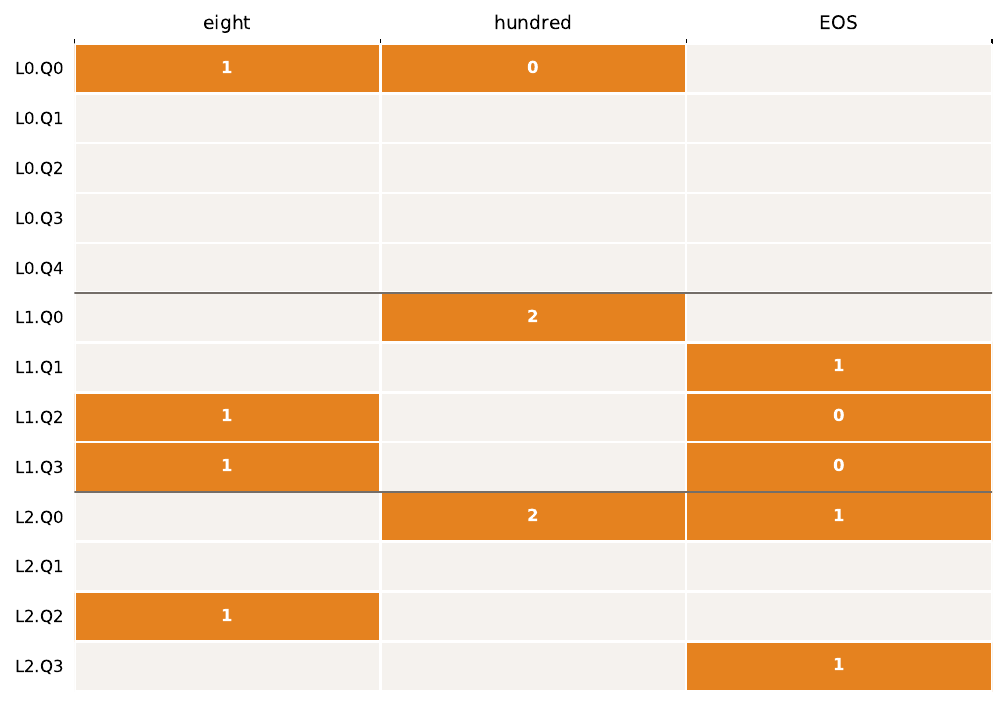}
    \caption{\small \textbf{A second Q-model execution, $800$.} Orange cells mark active quantum gates; the number inside is the count of active scalar writers in that gate's three-writer bank. Blank cells are inactive gates. Each column is one greedily emitted token. The source and Q-model both match the target at all three steps.}
    \label{fig:qmodel_800_trace}
\end{figure}

\spar{Graph and ordering diagnostics}
The fitted message graph is not a unique source circuit. The selected Q-model has four hard channels; graph stability across matched global-budget refits remains to be established. The frequency/complexity inset instead uses a distinct $1\%$ diagnostic factorization with $22$ factors to resolve the curve. For each factor, the source Transformer’s MLP-block updates are averaged using its normalized temporal-priority weights. The displayed line smooths each layer over $375$ source steps before averaging layers. Repeating the factorization at $1\%$, $3\%$, and $5\%$ gives pooled rank correlations of $0.678$, $0.805$, and $0.868$ between acquisition time and packet complexity, and $-0.774$, $-0.782$, and $-0.735$ between acquisition time and demand divided by packet complexity. All nine layer-by-threshold comparisons have the predicted signs. Because the refits share one trajectory and their factors are not independent, this is a threshold-robust consistency check, not an independent measurement of $\kappa_q(I)$.

\subsection{Ground-Truth Number-Naming Q-Program}
\label{app:qprogram_code}

The program below shows the core prediction routine and representative quantum
annotations. The omitted wrappers have the same form: each assigns a stable
identity to one deterministic operation such as selecting a chunk, reading a
digit, or choosing a lexical form.

{\scriptsize
\begin{verbatim}
@registry.quantum("TAIL_KIND", output_cardinality=6)
def q_tail_kind(tens_digit, unit_digit):
    return tail_kind(tens_digit, unit_digit)

@registry.quantum("BOUNDARY_ACTION", output_cardinality=3)
def q_boundary_action(group, position, content_length):
    return boundary_action(group, position, content_length)

def predict(state):
    high_length = q_high_content_length(state)
    group = q_control_group(state, high_length)
    chunk = q_select_chunk(state, group)
    position = q_group_progress(state, group, high_length)
    hundreds = q_hundreds_component(chunk)
    tens = q_tail_tens_digit(chunk)
    unit = q_tail_unit_digit(chunk)
    kind = q_tail_kind(tens, unit)
    length = q_chunk_length(hundreds, kind)
    boundary = q_boundary_action(group, position, length)

    with when(boundary, BoundaryAction.THOUSAND) as active:
        if active:
            return q_emit_thousand(boundary)
    with when(boundary, BoundaryAction.EOS) as active:
        if active:
            return q_emit_eos(boundary)
    with when(boundary, BoundaryAction.CONTENT) as active:
        if active:
            slot = q_content_slot(position, hundreds)
            with when(slot, ContentSlot.HUNDRED_WORD) as word:
                if word:
                    return q_emit_hundred(slot)
            form = q_lexical_form(slot, kind)
            with when(form, LexicalForm.TEN) as ten:
                if ten:
                    return q_emit_ten(form)
            value = q_lexical_value(
                slot, hundreds, tens, unit)
            with when(form, LexicalForm.UNIT) as unit_form:
                if unit_form:
                    return q_lex_unit(value)
            with when(form, LexicalForm.TEEN) as teen:
                if teen:
                    return q_lex_teen(value)
            with when(form, LexicalForm.TENS) as tens_form:
                if tens_form:
                    return q_lex_tens(value)
    return q_emit_eos(boundary)
\end{verbatim}
}

\spar{Compilation and guidance}
The program marks typed Python subroutines with stable quantum identities.
The compiler validates the registry and entry point, then executes a separate
audit domain twice to check determinism. Semantic values carry data-producer
identities, while explicit branch guards record control parents. The compiler
checks stable parent signatures, a data or control path from every active
quantum to the emitted token, and observed counterfactual effects of annotated
primitives. It unions the witnessed dependencies and takes their transitive
reduction. Training and evaluation events are then traced against this frozen
graph to provide active-node and ancestor-eligibility masks, semantic values,
and the single emitting quantum. These are labels for the declared program,
not labels inferred from the Transformer; unseen dependencies in a new input
domain require a separate coverage check.

The compiled graph fixes node order, direct Hasse edges, topological depth,
eligibility, and output leaves for the neural Q-model. In the factorized
Number-Naming program it contains $20$ quanta, $25$ direct edges, and $10$
levels. At each prediction event, the learned model reads causally visible
token memory and combines the initial residual with each unique ancestor
message once, including at graph joins. Local gate probabilities are
downward-closed under predicted routing by taking the minimum with parent
activity; the activity multiplier is detached from task-loss gradients, while
eligible gates receive balanced binary-cross-entropy supervision from the
compiled traces. The readout sums the initial residual and messages from
compiler-declared leaves; intermediate nodes influence it through descendants.
Semantic-value classifiers supervise selected message deltas during training
and are removed for evaluation. Reported generalization uses predicted routing;
oracle routing is a separate
diagnostic, so exact program execution does not imply exact neural execution.

The factorized compiled Q-model has 20 quanta over ten execution levels and
$53{,}728$ inference parameters, versus $46{,}284$ for the five-layer comparison
Transformer; its gates, messages, and readout remain learned, and the
optimization conditions of this sanity comparison are not matched. For
guidance, the frozen Q-model sums its executed messages at each depth to form
$Z_d$. A fixed semi-orthogonal projection maps the corresponding Transformer
block update into Q-space; the guidance objective in
Section~\ref{sec:number_naming} uses unit depth scales and alignment weight
one. The aligned and cross-entropy-only $10$-block
Transformers have $89{,}004$ deployment parameters and hold architecture,
data, schedule, initialization, and seed fixed; their alignment weights are one
and zero, respectively. The saved aligned Transformer records no Q dependency
at inference. Depth-shuffled and example-shuffled targets, together with
additional seeds, remain specificity and stability controls rather than
ingredients of the reported model.

\end{document}